%% file: Example.tex
\documentclass[a4paper,twoside]{article}

\usepackage{epsfig}
\usepackage{subcaption}
\usepackage{calc}
\usepackage{amssymb}
\usepackage{amstext}
\usepackage{amsmath}
\usepackage{amsthm}
\usepackage{multicol}
\usepackage{pslatex}
\usepackage{natbib}
\let\bibhang\relax
\usepackage{example}
\usepackage[bottom]{footmisc}

\usepackage{mathtools}
\usepackage{xr}
\usepackage{svg}
\usepackage[ruled,linesnumbered]{algorithm2e}
\newtheorem{theorem}{Theorem}
\RestyleAlgo{ruled}
\usepackage{booktabs}
\usepackage{amsfonts}
\usepackage{multicol}
\usepackage{wrapfig}

\graphicspath{{./fig/}}
\theoremstyle{definition}
\newtheorem{definition}{Definition}
\usepackage{hyperref} 
\usepackage{SCITEPRESS}     % Please add other packages that you may need BEFORE the SCITEPRESS.sty package.

\begin{document}

\title{FInC Flow: Fast and Invertible $k \times k$ Convolutions for Normalizing Flows}

% \author{\authorname{First Author Name\sup{1}\orcidAuthor{0000-0000-0000-0000}, Second Author Name\sup{1}\orcidAuthor{0000-0000-0000-0000} and Third Author Name\sup{2}\orcidAuthor{0000-0000-0000-0000}}
% \affiliation{\sup{1}Institute of Problem Solving, XYZ University, My Street, MyTown, MyCountry}
% \affiliation{\sup{2}Department of Computing, Main University, MySecondTown, MyCountry}
% \email{\{f\_author, s\_author\}@ips.xyz.edu, t\_author@dc.mu.edu}
% }

\author{\authorname{Aditya Kallappa\sup{1}, Sandeep Nagar\sup{1} and Girish Varma\sup{1}}
\affiliation{\sup{1}International Institute of Information Technology, Hyderabad, India}
\email{\{aditya.kallappa, sandeep.nagar\}@research.iiit.ac.in, girish.varma@iiit.ac.in}}

\keywords{Normalizing Flows, Deep Learning, Invertible Convolutions}

\abstract{Invertible convolutions have been an essential element for building expressive normalizing flow-based generative models since their introduction in Glow. Several attempts have been made to design invertible $k \times k$ convolutions that are efficient in training and sampling passes. Though these attempts have improved the expressivity and sampling efficiency, they severely lagged behind Glow which used only $1 \times 1$ convolutions in terms of sampling time. Also, many of the approaches mask a large number of parameters of the underlying convolution, resulting in lower expressivity on a fixed run-time budget. We propose a $k \times k$ convolutional layer and Deep Normalizing Flow architecture which i.)  has a fast parallel inversion algorithm with running time O$(n k^2)$ ($n$ is height and width of the input image and k is kernel size), ii.) masks the minimal amount of learnable parameters in a layer. iii.) gives better forward pass and sampling times comparable to other $k \times k$ convolution-based models on real-world benchmarks. We provide an implementation of the proposed parallel algorithm for sampling using our invertible convolutions on GPUs. Benchmarks on CIFAR-10, ImageNet, and CelebA datasets show comparable performance to previous works regarding bits per dimension while significantly improving the sampling time.}

\onecolumn \maketitle \normalsize \setcounter{footnote}{0} \vfill

\input{paper_tex/1_introduction}

\input{paper_tex/2_rel_work.tex}
\input{paper_tex/3_approach.tex}

\input{paper_tex/4_results.tex}
\input{paper_tex/5_conclusion.tex}

\bibliography{Example}
\bibliographystyle{Example}

\input{paper_tex/appendix.tex}

% \section*{\uppercase{Appendix}}

% If any, the appendix should appear directly after the
% references without numbering, and not on a new page. To do so please use the following command:
% \textit{$\backslash$section*\{APPENDIX\}}

\end{document}

%% file: paper_tex/1_introduction.tex
\section{Introduction}

% In contrast to discriminative models, generative models learn a representation of the dataset $p(x)$, whilst discriminative models learn a probability distribution of the dataset $p(y|x)$. This modeling technique is applicable to a variety of applications, including super-resolution, denoising, and inpainting. The work of generative modeling is difficult and computationally expensive since datasets are frequently of high dimension. In contrast to likelihood-based approaches, Generative Adversarial Networks (GANs) \citep{goodfellow2014generative} are recognized for their capacity to generate large and realistic pictures \citep{karras2017progressive}. The disadvantages of GANs include their general absence of latent-space encoders, their general lack of full data support, the complexity of optimization, and the difficulty of assessing overfitting and generalization \citep{grover2018creating}.
% \begin{wrapfigure}{r}{0.4\linewidth}
%     \centering
%     \includegraphics[width=0.99\linewidth]{fig/celeba_recon.pdf}
%     \caption{Reconstructed face images after training our FInC Flow model on the CelebA (64x64).}
%     \label{fig:celeba_reconstruct}
%     \vspace{-01em}
% \end{wrapfigure}

Normalizing flow is an important subclass of Deep Generative Models that offers distinctive benefits \citep{kobyzev2020normalizing}. In comparison to GANs \citep{gans} and VAEs \citep{kingma2018variational}, they are trained using a very intuitive Maximum Likelihood loss function. Images and the \emph{latent vector}, which is required to have a Gaussian distribution, correspond one-to-one in flow models. Despite these intriguing characteristics, GANs and VAEs are utilized more frequently. This is due to the need for the Normalizing Flows transformations to be invertible, which significantly restricts the neural network types employed. For deployment in a real-world scenario, the invertible transformations must be efficiently calculable in the forward and sample stages.

% intro to glow, invertible convs, issues of forward and sampling passes, state the problem of kxk invertible convolutions
A significant breakthrough came with Glow \citep{kingma2018glow} which used $1\times 1$ invertible convolutions to design normalizing flows. If it exists, the inverse function for a $1 \times 1$ convolution also happens to be a $1\times 1$ convolution. Since computing $1 \times 1$ convolution has fast parallel algorithms for which running time does not depend on the spatial dimensions, they are also highly efficient in forward pass (i.e. computing \emph{latent vector} from an image) as well as the sampling passes (i.e. computing image from a sampled \emph{latent vector}). Extending Glow to use invertible $k\times k$ convolutions promises to improve the expressivity further, allowing it to model more complex datasets. However, this is a challenging problem since the inverse function for a $k \times k$ convolution, in general, is given by a $n^{2} \times n^{2}$ matrix where $n = H = W$ (i.e. the spatial dimensions). Hence, while the forward pass can be fast, the trivial approach for the sampling pass will cost $O(n^4)$ operations per convolutional layer.

\begin{figure*}[!th]
    \centering
    \includegraphics[width=\linewidth]{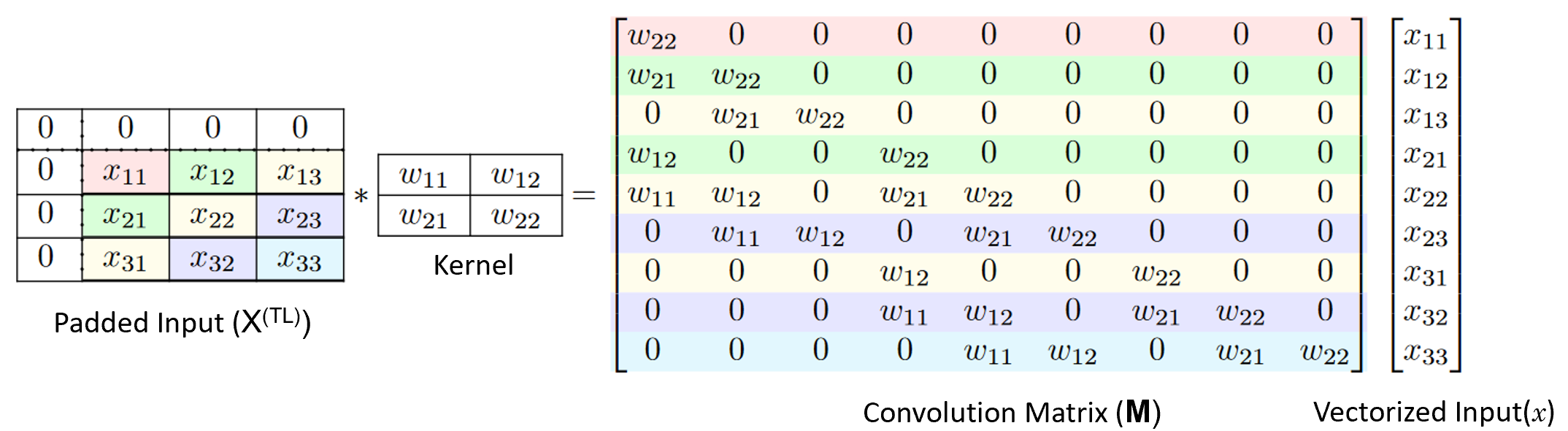}
    \caption{Convolution of a $3 \times 3$ \emph{TL} (Top Left) padded image with a $2 \times 2$ filter viewed as a linear transform of vectorized input($x$) by the convolution matrix $\textbf{M}$. The \emph{TL} padding on the input results in the making matrix $\textbf{M}$ lower triangular, and all diagonal values correspond to $w_{k,k}$ of the filter. Each row of $\textbf{M}$ can be used to find a pixel value. The rows or pixels with the same color can be inverted in parallel since all the other values required for computing them will already be available at a step of our inversion algorithm \ref{alg:inv_algo}.}
    \label{fig:Conv_equn}
\end{figure*}

% TODO: previous approaches to designing invertible kxk convs, their drawbacks
% There have been several attempts at designing invertible $k \times k$ convolutions that are efficient to compute in the forward and sampling passes. Emerging convolution \citep{hoogeboom2019emerging}, extended work of Glow proposed the chaining of autoregressive convolution and requires two convolutions for the $n \times n$ convolution. \citep{hoogeboom2020convolution}  introduced a method for the invertible transformation by converting a linear transformation to exponential transformation and then doing the computation of Jacobin's determinant, but the sampling time is slow. 

CInC Flow \citep{nagar2021cinc} introduced a padded $3 \times 3$ convolution layer design and gave it the necessary and sufficient conditions to make it invertible. They showed that the convolution matrix is lower triangular by ensuring padding in only two sides of the input. Furthermore, all the diagonal entries of the convolution matrix are equal to a single weight parameter. By setting this parameter to 1, they ensured that the convolutions are invertible, and Jacobian is always 1.

We build on their work by proposing a parallel inversion algorithm for their convolution design. The parallel algorithm only uses  O$(n k^2)$ sequential operations, unlike O$(n^2 k^2)$ operations used by most previous works. We also build a normalizing flow architecture, where channel-wise splitting is further used to parallelize operations.

\paragraph{Our Contributions.}
\begin{enumerate}
    \item We design a $k\times k$ invertible convolutional layer with a fast and parallel invertible sampling algorithm (see Sections \ref{sec:conv-design}, \ref{subsec:PIA}).
    \item We build a normalizing flow architecture based on the fast invertible convolution, which uses channel wise splitting to improve the parallelism further (see Sections \ref{section:fastflowunit}, \ref{subsec:arch}).
    \item We provide a fast GPU implementation of our parallel inversion algorithm and benchmark the sampling times of the model (see Section \ref{sec:results}). We show greatly improved sampling times due to our parallel inversion algorithm, while giving similar bits per dimensions as compared to other works.
\end{enumerate}

%% file: paper_tex/2_rel_work.tex
\begin{table*}[!t]
    \centering
    \begin{tabular}{l  l  l  l  l }
        \toprule
        Method    & \# of ops & \# params / CNN layer & \ Complexity of Jacobian & Inverse\\ 
        \midrule
        FInC Flow (our)  &    $(2n-1)k^2$ & $k^2 - 1$    & 1 &     exact   \\
        Woodbury \citep{lu2020woodbury} & $cn^2$ & $k^2$ & O$(d^2(c + n) + d^3)$ & exact \\
        MaCow \citep{ma2019macow}     &  $4nk^2$   & $k( \lceil \frac{k}{2} \rceil -1$) &  O$(n^3)$ & exact\\ 
        Emerging \citep{hoogeboom2019emerging}  &  $2n^2k^2$ &  $k( \lceil \frac{k}{2} \rceil -1$)  &  O$(n)$  &    exact   \\
        CInC Flow \citep{nagar2021cinc}  &  $n^2 k^2$   & $k^2 - 1$    & 1  &   exact      \\
        \midrule
        MintNet   \citep{song2019mintnet} &  $3n$   &  $ \frac{k^2}{3}$   &  O$(n)$ &    approx  \\ 
        SNF  \citep{keller2021self} &   $k^2$   &  $k^2$  &  approx &  approx \\ \bottomrule
    \end{tabular}
    \caption{Comparison of the learnable parameters. where $ n \times n $ is input size, $k \times k$ is filter size which is constant, $c$ is number of input/output channels. $d$ is the number of latent dimensions. \# of ops: required number of operations for the inversion of convolutional layers. The complexity of Jacobian: Time complexity for calculating the Jacobian of a single convolution layer. For FInC Flow and CInC Flow, the Jacobian is 1, since the Convolution matrix is lower triangular with diagonal entries being 1.}
    \label{table:complexity}
\end{table*}

\section{Related Work}
\paragraph{Generative Modeling.}
The idea of generative modeling stems from training a generative model whose sample comes from the same distribution as the training data distribution. Most of the generative models can be grouped as Generative adversarial networks (GANs) \citep{goodfellow2014generative, brock2018large}, Energy-based models (EBMs) \citep{zhang2022generative, song2021maximum}, Variational autoencoders (VAEs) \citep{kingma2013auto, kingma2018variational, hazami2022efficient}, Autoregressive models \citep{oord2016wavenet, nash2020polygen}, Diffusion models \citep{ho2020denoising, song2021scorebased, song2019generative} and Flow-based models \citep{dinh2014nice, dinh2016density, hoogeboom2019emerging,kingma2018glow, ho2019flow++, ma2019macow, nagar2021cinc}. 

\paragraph{Normalizing Flows.}
Flows-based models construct complex distributions by transforming a probability density through a series of invertible mappings \citep{rezende2015variational}. At the end of these invertible mapping, we obtain a valid distribution; hence, this type of flow is referred to as a Normalizing Flow model. Flow models apply the rule for change of variables; the initial density ‘flows’ through the sequence of invertible mappings \citep{dinh2016density}. Flow-based models generalize a dataset distribution into a latent space \citep{kobyzev2020normalizing}.

\paragraph{Invertible kxk Convolutions.}
% two types of invertible convolutions. Methods with exact inverses
An invertible neural network requires the inverse of the network with fast and efficient computation of the Jacobian determinant \citep{song2019mintnet}. An invertible neural network can be used for generation and classification with more interpretability. \citep{kingma2018glow} proposed an invertible $1 \times 1$ convolution building on top of NICE \citep{dinh2014nice}  and RealNVP \citep{dinh2016density} consisting a series of flow step combined in a multi-scale architecture. Each flow step consists of actnorm followed by an invertible $1 \times 1$ convolution, followed by a coupling layer (see Sec \ref{section:fastflowunit}). Emerging \citep{hoogeboom2019emerging} presented method to generalized $1 \times 1$ convolution to invertible $ k \times k$ convolutions. Emerging chains two specific auto regressive convolutions \citep{kingma2013auto} to form a single convolutional layer following the associativity of the convolution operation. Each of these autoregressive convolutions is chosen such that the resulting convolution matrix $\mathbf{M}$ is triangular with an inverse time of each of the convolutions is O$(n  \times n \times k^2)$. MintNet \citep{song2019mintnet} presented a method for designing invertible neural networks by combining building blocks with a set of composition rules. The inversion of the proposed blocks necessitates a sequence of dependent computations that increase the network's sampling time. SNF \citep{keller2021self} proposed a method to reduce the computation complexity of the Jacobian determinant by replacing the gradient term with a learned approximate inverse for each layer. This method avoids the determinant of Jacobian and makes it approximate, and requires an additional backward pass for inversion of convolution. MaCow \citep{ma2019macow} while many other papers make use of the invertibility of triangular matrix to reduce inversion time, MaCow outperforms all of them by performing the inverse in O$(n k^2)$ by carefully masking $4$ kernels at the top, left, bottom, right to achieve a full convolution, but this flow model use four autoregressive convolutions to make an effective standard convolution. Woodbury \citep{lu2020woodbury} this paper employs the \emph{Woodbury transformation} for invertible convolution, which is a generalized permutation layer that models dimension dependencies along the channel and spatial axes using the channel and spatial transformation. ButterflyFlow \citep{meng2022butterflyflow} introduced a new family of an invertible layer that works for special underlying structures and needs a sequence of layers for an effective invertible convolution.
%  write more 

\paragraph{Fast Algorithms for Invertible Convolutions.}
CInC Flow \citep{nagar2021cinc}, derive necessary and sufficient conditions on a padded CNN for it to be invertible and require a single CNN layer for every effective invertible CNN layer. The padded CNN can leverage the advantage of parallel computation for inversion, resulting in faster and more efficient computation of Jacobian determinants.

The distinguishing feature of our invertible convolutions as compared to previous works is that we have a parallel inversion algorithm that does only $(2n-1)k^2$ operations where $n$ is input size and $k$ is kernel size. MaCow is the closest approach that takes twice the number of operations. Some of the approaches, like MintNet and SNF, do achieve a lesser number of operations. However, they are not proper normalizing flows as they compute only an approximate inverse. We use the convolution design from CInC Flow but give a parallel inversion algorithm for it. Furthermore, our FInC Flow \emph{Unit} is designed to efficiently parallelize the operations by splitting the convolution operations channel-wise. In Table \ref{table:complexity}, we compare our proposed flow model with the existing model in terms of the receptive fields/number of learnable parameters, complexity of computing the inverse of convolution layer for sampling.

%%%%%%%%%%%%%%%%%%%%%%%%%%%%%%
%           TRASH
%%%%%%%%%%%%%%%%%%%%%%%%%%%%%%

% Woodbery Tranforms, MaCow,

% Normalizing flows have also been used for density estimation, variational inference, and generative modeling. Previous work on Normalizing flow, Non-linear Independent Component Estimation (NICE) \citep{dinh2014nice} to apply flow-based models for modeling complex high-dimensional densities. RealNVP \citep{dinh2016density} extended NICE with a more flexible invertible transformation to experiment with images. Glow \citep{kingma2018glow} as a generative flow with invertible $ 1 \times 1 $ convolutions and significantly improved the density estimation performance. The extension of Glow, Emerging \citep{hoogeboom2019emerging} generalize  $ 1 \times 1 $ to  $ d \times d $ by padding the input and masking the kernel, but for one standard convolution, Emerging need two convolutions. 

% Let $ f: \mathbb{R} ^{d} \rightarrow \mathbb{R} ^{d} $ with inverse $ f^{-1} = g $, i.e. the composition $g \cdot f(z)$, where $z$ is the latent vector (a random variable) with distribution $q(z)$, resulting random variables $z^{'} = f(z) $ has a distribution: 
% \begin{equation}
%     q (z^{'}) = q(z) \left| \det \frac{\delta f^{-1}}{\delta z^{'}}\right| = q(z) \left| \det \frac{\delta f}{ \delta z}\right|^{-1}
% \end{equation}

% where right equality can be seen be applying the chain rule and its property of \emph{Jacobian} of invertible function. 

%% file: paper_tex/3_approach.tex
\section{Preliminaries}\label{sec:FastFlow}
\paragraph{Normalizing Flows.}
% Define what they are.
Formally, Normalizing Flows is a series of transformations of a known simple probability density into a much more complex probability density  using invertible and differentiable functions. These invertible function allows to write the probability of the output as a differentiable function of the model parameters. As a result, the models can be trained using backpropagation with the negative log likelihood loss function.

Let $\mathbf{X} \in \mathbb{R}^d$ be a random variable with tractable density $p_\mathbf{X}$. Let $f:\mathbb{R}^d \rightarrow \mathbb{R}^{d}$ be a differentiable and invertible function. If $\mathbf{Y} = f(\mathbf{X})$ then the density of $Y$ can be calculated as 
$$ p_X(x) = p_Y(y) \left| \det J_f  \right| \qquad \text{ where } \qquad J_f =  \frac{\partial f(x)}{\partial x}.$$ 

Note that $J_f$ is a $d \times d$ matrix called the Jacobian. If $X$ is transformed using a sequence of functions $f_i$'s. That is
$f = f_1 \circ f_2 \circ f_3 \circ \cdots \circ f_r$. Now probability density, $p_Y(y)$ can be expressed as 
\begin{equation}
    p_Y(y) = p_X(f^{-1}(y)) \cdot \displaystyle \prod_{i = r}^{1}  |J_{f_i^{-1}}(y_i)| .
\end{equation}
where $ y_i = f_i^{-1}\circ \cdots \circ f_r^{-1}(x)$.
The log-probability of $p_Y$ which will be used to model the complex image distribution is given by,
\begin{equation}
\label{eq:nll}
    \log p_Y(y) = \log p_X(f^{-1}(y_r)) + \displaystyle \sum_{i=1}^{r}  \log |\det 
    J_{f^{-1}_i}(y_i)|.
\end{equation}
The functions $f_i^{-1}$ will be given by neural network layers and the above function can be computed during the forward pass of the neural network. The negative of this function called the negative log likelihood (NLL) is minimized when images in the dataset are given highest probabilities. Hence it gives a simple, interpretable loss function for training the model. 

\begin{figure*}[!ht]
    \centering
    \includegraphics[width=\textwidth]{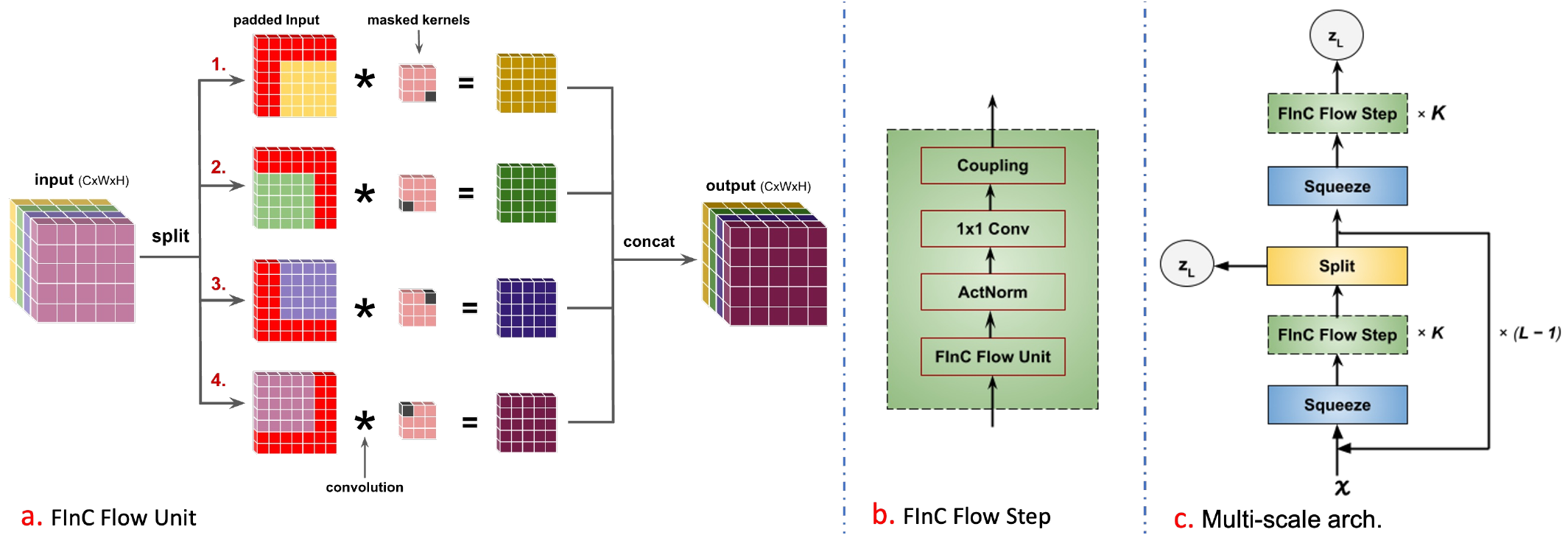}
    \caption{(a) FInC Flow unit: to utilize the independence of convolution on channels the input channels are sliced into four equal parts and then padded (1. top-left, 2. top-right, 3. bottom-right, 4. bottom-left) to keep the input size and output size same. Next, parallelly convoluted each sliced channel with the corresponding masked kernel (masked corner of kernels: 1. bottom-right, 2. bottom-left, 3. top-left, 4. top-right). Finally, concatenate the output from each convolution. (b) We propose a FInC Flow architecture  \ref{section:fastflowunit} where each FInC Flow \emph{Step}  consists of an actnorm step, followed by an invertible 1 × 1 convolution, followed by coupling layer. (c) Flow is combined with a multi-scale architecture \ref{subsec:arch}.}
    \label{fastflow}
\end{figure*}

\paragraph{Invertible Convolutions.}

The convolution of an input $X$ with shape $H \times W \times C$ with a kernel $K$ with shape $k \times k \times C \times C$ is $Y = X \times K$ of shape $(H - (k+1)) \times (W -  (k+1)) \times C$ which is equal to
\begin{equation}
Y_{i,j,c_o} = \sum_{l,h < k} \sum_{c_i=1}^{C} X_{i+l,j+k,c_i}K_{l,k,c_i,c_o} 
\end{equation}
Notice that the dimensions of $X$ and $Y$ are not necessarily the same. To ensure that the $X$ and $Y$ are the same size, we apply padding to the input $X$. For an input image $X$ with shape $H \times W \times C$, the $(t,b,l,r)$ padding of $X$ is the image $\hat{X}$ of shape $(H+t+b) \times (W+l+r) \times C$ is defined as
\begin{equation}
\hat{X}_{i, j, c} = 
    \begin{cases}
        X_{i-t, j-l, c} & \text{if } i - t < H \text{ and } j - l < W \\
        0 & \text{otherwise}
    \end{cases}
\end{equation}
The convolution operation is a linear transformation of the input. For the vectored (flattened) input $X$, denoted by $x$, the output vector $y$ can be written as  $y = \mathbf{M} x$. Matrix $\mathbf{M}$ is the called as the \emph{Convolution Matrix} and the dimensions of this matrix is $HWC \times HWC$. As long as matrix $\mathbf{M}$ is invertible, the convolutional layer can be included as a part of the Normalizing Flows. The common approach to building invertible convolutions is by making $M$ upper triangular and ensuring invertibility by making diagonal entries to be nonzero.

\paragraph{Algorithms for Computing Inverse of Convolutions.}
For normalizing flows built using invertible convolutions, the sampling pass will involve computing the inverse of the convolution matrix. This involves solving a linear systems of equations $\textbf{M}x =y$.

For a general square matrix of size $n \times n$, the time complexity for inversion is $O( n^3)$. For a lower triangular matrix of size $n \times n$, the time complexity for inversion is $O(n^2)$ because of back-substitution method. Notice that the size of convolution matrix, \textbf{M}) is $n^2 \times n^2$ (refer Figure \ref{fig:Conv_equn}) and also that row of the matrix has only $k^2$ entries at the maximum, results in an inversion time of $O(n^{2} k^{2})$ which is used in many of the previous works like Emerging and CInC Flows. We show that this method can be parallelized for carefully designed convolutions giving a complexity of only $O(n k^2)$.

\section{FInC Flow: Our Approach}
In this section we describe our approach including convolution layer design which has a fast parallel inversion algorithm with running time O$(n k^2)$. For more clarity, we refer to height of the image as $H$, width as $W$ and channels as $C$ in this section.

\subsection{Convolution Design} \label{sec:conv-design}
As it is obvious from equation $ x = \textbf{M}^{-1} y$, the inverse timings depends on $\mathbf{M}$. Emerging \citep{behrmann2019invertible} masks almost half of the convolution kernel values to ensure $\textbf{M}$ is a Lower Triangular Matrix. However, we follow the method followed in CInC Flow, where only a few values of the convolution kernel are masked.
% \noindent
For an input image $X$ with shape $H \times W \times C$, the top-left(\emph{TL}) i.e., $(t, 0, l, 0)$ padding of $X$ is the image $X^{\text{(TL)}}$ of shape $(H+t) \times (W+l) \times C$ is defined in equation \ref{eq:top_left_padding} and similarly for the top-right (\emph{TR}) as equation \ref{eq:top_right_padding}, bottom-left (\emph{BL}) as equation \ref{eq:bottom_left_padding}, bottom-right (\emph{BR}) as equation \ref{eq:bottom_right_padding}. 
\begin{align}
    X^{\text{(TL)}}_{i, j, c} &= 
    \begin{cases}
        X_{i-t, j-l, c} & i - t > 0 ~\wedge~ 
        j - l > 0 \\
        0 &\text{otherwise}
    \end{cases} \label{eq:top_left_padding} \\
    X^{\text{(TR)}}_{i, j, c} &= 
    \begin{cases}
        X_{i-t, j, c}~~~ &  i - t > 0 ~\wedge~ j - r < W \\
        0 &\text{otherwise}
    \end{cases} \label{eq:top_right_padding} \\
    X^{\text{(BL)}}_{i, j, c} &= 
    \begin{cases}
        X_{i, j-l, c}~~~~ &  i - b < H ~\wedge~ j - l > 0 \\
        0 &\text{otherwise}
    \end{cases} \label{eq:bottom_left_padding}\\
    X^{\text{(BR)}}_{i, j, c} &= 
    \begin{cases}
        X_{i, j, c}~~~~~~~~ &  i - b < H ~\wedge~ j - r < W \\
        0 &\text{otherwise}
    \end{cases}  \label{eq:bottom_right_padding}
\end{align}
Figure \ref{fig:Conv_equn} shows the convolution of a \emph{TL} padded $3 \times 3$ image with a $2 \times 2$ filter is equivalent to a matrix multiplication between convolution matrix $\textbf{M}$ and vectored input $x$.  We leverage this to find the inverse faster. We discuss this in more detail in the subsequent sections. Also padded input $X^{\text{(TR)}}$, $X^{\text{(BL)}}$ and $X^{\text{(BR)}}$ are equivalent to $X^{\text{(TL)}}$ once they are flipped along corresponding dimension(s).

\SetKwComment{Comment}{/* }{ */}

\begin{algorithm}[!t] 
    \caption{Fast Parallel Inversion Algorithm of \emph{TL} padded convolution block(PCB)}
    \label{alg:inv_algo}
    \KwIn{$K$: Kernel of shape $(C, C, k_H, k_W)$\\
    $Y$: output of the conv of shape $(C, H, W)$}
    \KwResult{$X$: inverse of the conv. with shape $(C, H, W)$.}
    \textbf{Initialization:} $X \gets Y$ \;
    \For{$d \gets 0, H + W - 1$}{
        \For {$ c \gets 0, C - 1 $}{
            \Comment{The below lines of code executes parallelly on different threads on GPU for every index $(c, h, w)$ of $X$  on the $d$th diagonal.} 
            
                \For{$ k_h \gets 0, k_H - 1 $}{
                \For{$ k_w \gets 0, k_W - 1 $}{
                \For{$ k_c \gets 0, C - 1 $}{
                    \If{pixel $ (k_c, h - k_h, w - k_w) $ not out of bounds}{
                         $ X[c, h, w]  \gets X[c, h, w] - X[k_c,h - k_h, w - k_w] * K[c, k_c, k_H - k_h - 1, k_W - k_w - 1]; $
                      }
                }
                }
                } \Comment{synchronize all threads}
        } 
     }
\end{algorithm}

\subsection{Parallel Inversion Algorithm}\label{subsec:PIA}
We have presented our algorithm in Algorithm \ref{alg:inv_algo}. The algorithm can be understood using Figure \ref{fig:Conv_equn}.

\begin{definition}[Diagonal Elements]
Two pixels $x_{i, j}$ and $x_{i', j'}$ are said to be secondary diagonal elements if $i + i' = j+j'$ . For brevity, we refer to these elements from here on simply as Diagonal Elements.
\end{definition}

Theorem \ref{th:ind_parallel} proves that every element of on the diagonal can be computed parallelly and Line 2 of the algorithm takes care of that. We initialize $X$ to $Y$ in Line 1 and compute $X$ in Line 8 which is given in Equation \ref{eq:inv_convolution_imp}. 
It is important that we wait for the threads to synchronize before we move to the next diagonal, as they are needed for computing the elements of the next diagonal. The \emph{not out of bounds} in Line 7 means we are remaining in the $k \times k$ convolution window and also we are not including pixel $(i, j)$ while computing $x_{i, j}$ as given in Equation \ref{eq:inv_convolution_imp}

\begin{theorem}
\label{th:ind_parallel}
The inverse of the pixels on the diagonals of a TL padded convolution can be computed independently and parallelly.
\end{theorem}
\begin{proof}
The $(i, j)^{th}$ pixel value of the output $Y$ with shape $H \times W$ can be calculated as 
$$
y_{i, j} = (\textbf{M}_{iW+j, :})^T \cdot x
$$ 
which means $y_{i,j}$ is the dot product of $\textbf{M}_{iW+j, :}$ i.e., the corresponding row of matrix $\textbf{M}$ and the vectored input $x$.
Because it is a \emph{TL} padded convolution, $y_{i,j}$ depends only on the values of $k \times k$ window of $x_{\le i, \le j}$ pixels where $x_{\le i, \le j}$ are the pixels that are on the top and left side of the pixel $x_{i, j}$ including $x_{i, j}$. Because all the diagonal values are $w_{k, k}$, we have, 
\begin{align*}
y_{i, j} &= w_{k, k} x_{i, j} + f(x_{<i, <j})\\ 
x_{i, j} &= \frac{y_{i, j} - f(x_{<i, <j})}{w_{k, k}}   
\end{align*}

where $x_{<i, <j}$ are the pixels which are strictly top and left side of $(i, j)$. Following the masking pattern of CInC Flow, we have $w_{k, k} = 1$ and $f$ is a linear function which is given by weighted sum of the given pixels weighed by the filter values. So,
\begin{equation}
\label{eq:inv_convolution}
x_{i, j} = y_{i, j} - f(x_{<i, <j})
\end{equation}
\begin{equation}
\label{eq:inv_convolution_imp}
x_{i, j} = y_{i, j} - \sum_{p=0}^{k} \sum_{q=0}^{k} x_{i-p, j-q} K_{p, q}  \text{ where }  p = q \neq 0
\end{equation}

% \noindent
Let two pixels $x_{i, j}$ and $x_{i', j'}$ be  on the same diagonal. This also means that only one of the following settings is true a) $i < i'$ and $j > j'$ or b) $i > i'$ or $j < j'$. Either way, we can conclude that computation of $x_{i, j}$ is not dependent on $x_{i', j'}$ and vice versa following the result in Equation \ref{eq:inv_convolution}. Hence they can be computed independently.
% \noindent
Once $x_{i, j}$ is computed, following the Equation \ref{eq:inv_convolution} and the above result, we can compute $x_{i + 1, j}$ and $x_{i, j+1}$. Since, the sets of pixels $x_{<i+1, <j}$ and $x_{<i, <j+1}$ both include the elements of $x_{<i, j}$ and also $x_{i, j}$, we can write
\begin{align}
x_{i + 1, j} &= y_{i+1, j} - f(x_{<i+1, <j}) \nonumber\\
&=y_{i+1, j} - \alpha x_{i, j} - f_1(x_{<i, <j})\label{eq:parallel_1}\\
x_{i, j+1} &= y_{i, j+1} - f(x_{<i, <j+1}) \nonumber\\
&=y_{i, j+1} - \beta x_{i, j} - f_2(x_{<i, <j}) \label{eq:parallel_2}
\end{align}

where $\alpha$ and $\beta$ are kernel weights.

From Equations \ref{eq:parallel_1}  and \ref{eq:parallel_2}, we can conclude that $x_{i+1, j}$ and $x_{i, j+1}$ which are on the same diagonal can be calculated parallelly in a single step.
\end{proof}

\begin{theorem}
Algorithm \ref{alg:inv_algo} uses only $(H+W-1)k^2$ sequential operations.
\end{theorem}

\begin{proof}
We have proved in Theorem \ref{th:ind_parallel} that the inverse pixels on a single diagonal can be computed parallelly in one iteration of Algorithm \ref{alg:inv_algo}. Since there are $H + W - 1$ number of diagonals in a matrix and there are at maximum $k^2$ entries in a row of the convolutional matrix, the number of sequential operations needed will be $(H+W-1)k^2$.
\end{proof}

Thus the running time of our algorithm is O$(n k^2)$ where $n = H = W$ 

\begin{table*}[!t]
\caption{Comparison of the  bits per dimension (BPD), forward pass time (FT) and sampling time (ST) on standard benchmark datasets of various $k \times k$ convolution based Normalizing Flow models. FT and ST are presented in seconds.}
\centering
       \resizebox{\textwidth}{!}{% <------ Don't forget this %
       
 \begin{tabular}{lcccccccccccc}
        \toprule
        \multicolumn{1}{l}{Model} & \multicolumn{3}{c}{MNIST} & \multicolumn{3}{c}{CIFAR-10} & \multicolumn{3}{c}{Imagenet-32x32} & \multicolumn{3}{c}{Imagenet-64x64}\\
        \cmidrule(lr){2-4} \cmidrule(lr){5-7} \cmidrule(lr){8-10} \cmidrule(lr){11-13}
           &   BPD &  FT & ST &  BPD &  FT & ST & BPD &  FT  &  ST   &   BPD  & FT & ST\\
        \midrule
        % Glow    &   1.05 &  0.14 & 0.08 &  3.36 & 0.40 & 0.22 & 4.12 & 0.59 & 0.37   &   3.81  & 1.51 & 1.17\\
        % Flow++ \citep{ho2019flow++} &   -- &  -- & -- &  3.08 & 1.35 & 1.71 &  3.86 &    2.75  &     2.49   &   3.69  &  16.47 &  1.89\\
        % SNF &  1.06 &  0.09 & 0.09 &  3.37 & 0.23 & 0.21 &  4.14 &    0.34  &     0.31   &    --  &  1.17 &  1.13\\
        % \hline
        Emerging &   -- &  0.16 & 0.62 &  3.34 & 0.49 & 17.19 &  4.09 &    0.73  &    25.79   &    3.81  &  1.71 &  137.04\\
        MaCow   &   -- &  --& -- &  3.16 & 1.49 & 3.23 &  -- &    --  &     --   &    3.69  &  2.91 &  8.05\\
        % Woodbury &   -- &  --& -- &  3.47 & -- & -- &  4.20 &    --  &     --   &    3.87  &  -- &  --\\
        CInC Flow & -- &  --& -- &  3.35 & 0.42 & 7.91 &  4.03 &   0.62  &     11.97   &    3.85  &  1.57 &  55.71\\
        MintNet & 0.98 &  0.16 & 17.29 &  3.32 & 2.09 & 230.17 &  4.06 &    2.08  &     230.44   &    --  &  -- &  --\\
        %WaveFlow & -- &  -- & -- &  -- & -- & -- &  4.08 &    --  &     --   &    3.78  &  -- &  --   \\%
        FInC Flow (our) &   1.05 &  0.14 & 0.09 &  3.39 & 0.37 & 0.41 &  4.13 &    0.48  &  0.52   &    3.88  &  1.43 & 2.11 \\  
        \bottomrule
        \end{tabular}}%
    
    \label{table:bpd_table}
\end{table*}

\subsection{FInC Flow Unit}\label{section:fastflowunit}

Figure \ref{fastflow}a visualizes our $k \times k$ convolution block. We call this block as FInC Flow \emph{Unit}. We use all the $4$ padding techniques mentioned before to different channels of the image. For this purpose, we split the input into four equal parts along the channel axis. We do \emph{TL} padding to the first part, \emph{TR} to the second part, \emph{BL} to the third part and \emph{TR} to the fourth part. Then we use a masked filter on each of these parts to perform the convolution operation parallelly. We call each of this padded image along with it's corresponding kernel as Padded Convolution Block (PCB).

\subsection{Architecture}\label{subsec:arch}
Figure \ref{fastflow}c shows the complete architecture of our model. Our model architecture resembles the architecture of Glow. The multi-scale architecture involves a block of a Squeeze layer, FInC Flow \emph{Step} repeated $K$ number of times and a Split layer. The whole block is repeated $L - 1$ number of times. A Squeeze layer follows this and finally FInC Flow \emph{Step} repeated $K$ times.
At the end of each split layer, half of the channels are 'split' (taken away) and modeled as Gaussian distribution samples. These splited half channels are \emph{latent vectors}. The same is done for the output channels. These are denoted as $z_L$ in Figure \ref{fastflow}(c). 
Each \emph{FInC Flow Step} consists of a \emph{FInC Flow Unit}, an Actnorm Layer, a $1 \times 1$ Convolutional Layer, followed by a coupling layer. \\
\textbf{Actnorm Layer:} Acts as an activation normalization layer similar to that of a batch normalization layer. Introduced in Glow, this layer performs the affine transformation using scale and bias parameters per channel. 
\\
\textbf{$1 \times 1$ Convolutional Layer:} This layer introduced in Glow does a $1 \times 1$ convolution for a given input. Its log determinant and inverse are very easy to compute. It also improves the effectiveness of coupling layers.
\\
\textbf{Coupling Layer:} RealNVP introduced a layer in which the input is split into two  half. The first half remains unchanged, and the second half is transformed and parameterized by the first half. The output is concatenation of first half and the affine transformation, by functions parameterized by the first, of second half. The inverse and log determinant of coupling layer are computed in a straightforward manner. Coupling layer consists of $3 \times 3$ convolution followed by a $1 \times 1$ and a modified $3 \times 3$ convolution used in Emerging.
\\
\textbf{Squeeze:} This layer takes features from spatial to channel dimension \citep{behrmann2019invertible}, i.e., it reduces the feature dimension by total four, two across the height dimension and two across the width dimension resulting in increases the channel dimension by four. As used by \citep{dinh2016density}, we use squeeze layer to reshape the feature maps to have smaller resolution but more channels.
\\
\textbf{Split:} Input is splited into two halves across the channel dimension. We retain the first half, and a function parameterized by first half transform the second half. The transformed second half is modeled as Gaussian samples, are the \emph{latent vectors}. We do not use the checkerboard pattern used in RealNVP \citep{dinh2016density} and many others to keep the architecture simple.

\SetKwComment{Comment}{/* }{ */}

%% file: paper_tex/4_results.tex
\begin{table*}[!t]
    \centering
    \caption{CIFAR-10: comparison of learnable parameters and the sampling time. FInC Flow has less number of learnable parameters with the same receptive field and fast layers (all the times are averaged over ten loops for n = 100 sample images in seconds). ST = Sample time, FT = Forward Time. MaCow-FG is the fine-grained MaCow model and MaCow-org stands for MaCow model utilizes the original multi-scale architecture which is the same as Glow. MaCow and our method is closely similar in term of the convolutional design. So, here we show that our proposed method do fast sampling while maintaining the faster forward time .}
        \begin{tabular}{l  c  c  c c}
        \toprule
            Models    & Setting (K and L) & Learnable params (M = million)  & FT(n=100) & ST(n=100) \\ 
            \midrule
            MaCow-FG & [4, [12, 12], [12, 12], 12]  &  37.19M   & 0.88  & 2.64  \\
            MaCow-org &  [4, [12, 12], [12, 12], 12], [4, 4]  & 38.4M & 1.48  & 3.23 \\ 
            % Glow  & [32, 32, 32]  &   44.7M  & 0.40  & 0.22    \\
            FInC Flow (our) &  [28, 28, 28] & 39.46M  & \textbf{0.37}  & \textbf{0.41}   \\
            \bottomrule
        \end{tabular}
    
    \label{table:params}
\end{table*}
\vspace{2em}

\section{Results}\label{sec:results}
% \paragraph{Negative log-likelihood / bits per dimension (BPD):}
\paragraph{Bits Per Dimension (BPD):}
BPD is closely related to NLLLoss given in equation \ref{eq:nll}. BPD of $H \times W \times C$ image is given by
\begin{equation}
\label{eq:bpd}
    \text{bpd} = \frac{\text{NLLLoss} \times\log_2{e}}{H W C} 
\end{equation}
% \noindent
Table \ref{table:bpd_table} shows the BPD comparative results of various models with our model. We present the results of MaCow-var which uses Variational Dequantization which was introduced in Flow++ \citep{ho2019flow++}. BPDs recorded are the reported numbers from the respective model papers.

\paragraph{Sampling Time:}\label{section:SamplingTimes}
\begin{figure*}[!ht]
    \centering
    \includegraphics[width=0.99\linewidth]{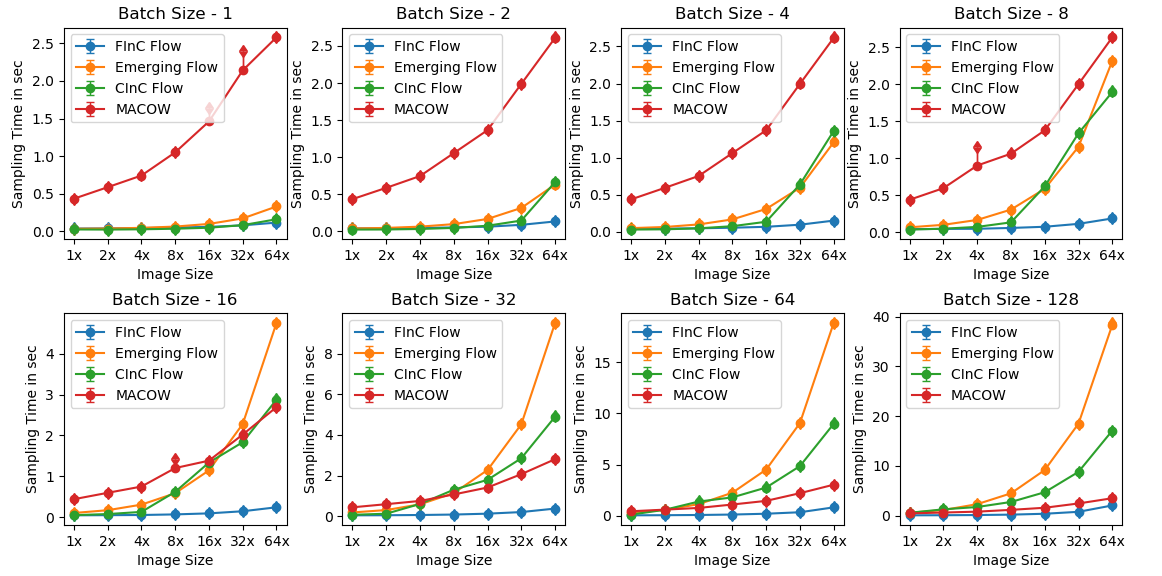}
    \caption{Sampling Times for four models - our, Emerging, CInC Flow,  MaCow. Each plot gives the $95\%$ Confidence Interval (CI) time of the ten runs to sample $100$ images. X-axis represents the sizes of the image sampled starting from $16 \times 16 \times 2$ ($H \times W \times C$) all the way to $128 \times 128 \times 2$. }
    \label{fig:timing}
\end{figure*}

% \textcolor{red}{From this figure, we can say that when we increase the batch size or image size, our method out perform the base line models for $k \times k$ inv. convolution in terms of sampling time.}}

\begin{figure*}[!ht]
    \centering
    \includegraphics[width=0.99\linewidth]{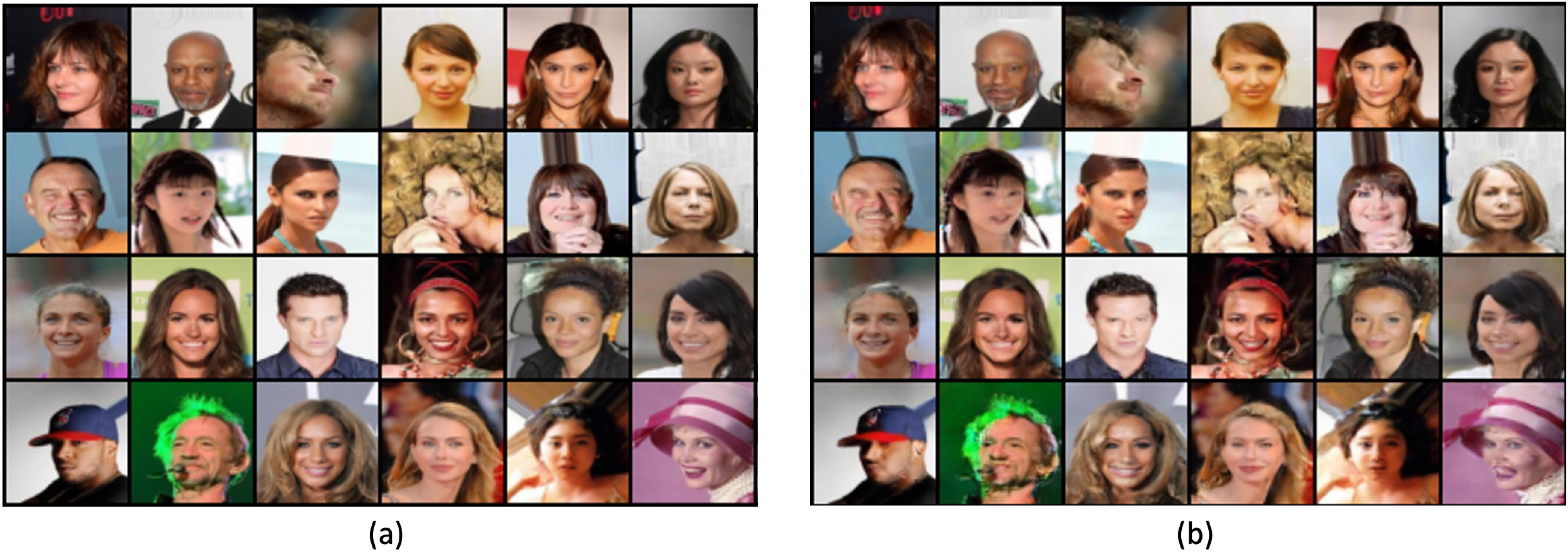}
    \caption{Comparison of  (a) original and (b) reconstructed image samples for the $64 \times 64$ CelebA dataset after FInC Flow model for $100$ epochs. From the images, we can conclude our model reconstruct original image.}
    \label{fig:celeba_reconstruct}
\end{figure*}
% \vspace{8em}
Table \ref{table:bpd_table} shows the comparative results of our model with other models. For MaCow, we use the official code released by the authors. We use the code for Emerging, which was implemented in PyTorch by the authors of SNF.. We have implemented CInC Flow in PyTorch and used it to generate results.
% We have implemented Emerging \citep{hoogeboom2019emerging} and CInC Flow \citep{nagar2021cinc} in PyTorch for the same. 
The FTs and STs are recorded by averaging ten runs on untrained models (including our model).

\begin{figure}[!ht]
    \centering
    \includegraphics[width=0.99\linewidth]{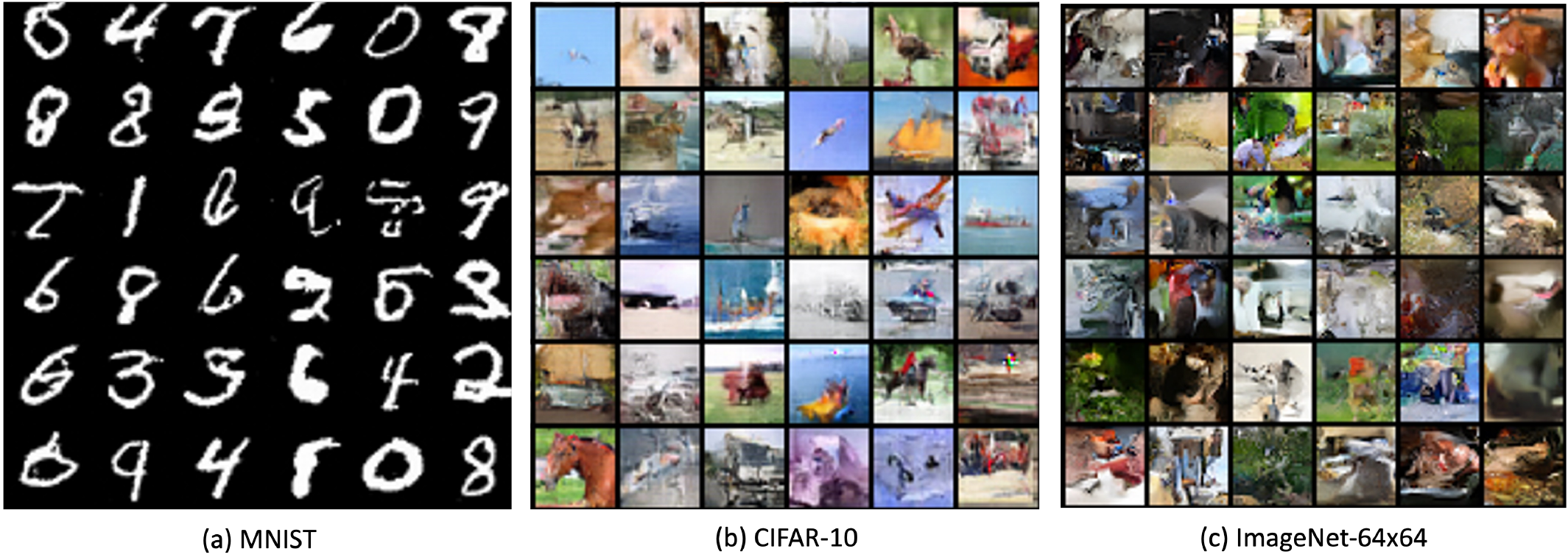}
    \caption{Uncurated generated samples images from our flow model.}
    \label{fig:uncurated_samples}
\end{figure}

In Figure \ref{fig:timing}, we plot the relationship between the input image size and inverse sampling time. As the input image size increase, our \emph{Parallel Inversion Algorithm} improve by utilizing the independence in the convolution matrix $M$. If we input a single image (batch size = $1$), our model performs similarly to the CInC Flow and Emerging. MaCow is far slower because it does the masking of four kernels to maintain the receptive field. To do one convolution, it needs four convolutions to complete one standard convolution, making it slower. Emerging requires two consecutive autoregressive convolutions to have the same receptive field as standard convolution and solver compared to FInC Flow. For batch size = $4$ and larger, FInC Flow beats the Emerging, MaCow , and CInC Flow by a big difference (see Figure \ref{fig:timing}) while maintaining the same receptive field.

\paragraph{Scaling sampling time with spatial dimensions:}
Table \ref{table:params} shows the comparison among the invertible convolution-based models. To keep it fair, we restrict the total parameters across all the models to be close to $5$ M. We note down the average sampling time (ST) to generate $100$ images over ten runs while doubling the size of the sampled image from $16 \times 16$ all the way to $128 \times 128$ and also doubling our batch size from $1$ all the way to $128$. Our model outperforms all the other models in most, if not all, the settings. All the models were untrained and run on a single NVIDIA GTX 1080Ti GPU.
\paragraph{Image reconstruction and generation:}
In Figure \ref{fig:celeba_reconstruct}, we present the effectiveness of the  FInC Flow model in the reconstruction (sampling) of the images. First, we feed the input image to forward flow and get the \emph{latent vector} ($z_L$). To reconstruct the images from the \emph{latent vector} ($z_L$), give the $z$ as input to the inverse flow. Figure \ref{fig:celeba_reconstruct} present the reconstructed face images for the CelebA dataset after training our model for $100$ epochs. The model takes a random sample from the Gaussian distribution for the \emph{latent vector} to generate sample images. This \emph{latent vector} is used to generate images by going backward in the flow model. In Figure \ref{fig:uncurated_samples}, we present generated sample by our model on the MNIST, CIFAR-10, and ImageNet-64x64 dataset.

%% file: paper_tex/5_conclusion.tex
\section{Conclusion}

With a parallel inversion approach, we present a $k \times k$ invertible convolution for Normalizing flow models. We utilize it to develop a model with highly efficient sampling pass, normalizing flow architecture. We implement our parallel algorithm on GPU and presented benchmarking results, which show a significant enhancement in forward and sampling speeds when compared to alternative methods for $k \times k$ invertible convolution.

% We present a $k\times k$ invertible convolution with a parallel inversion algorithm. We use it to build a normalizing flow architecture that is highly efficient in the sampling pass. We implemented our parallel algorithm on GPU and presented benchmarking results, which show significant improvement in sampling times as well as forward time compared to other  $k \times k$ invertible convolution approaches. 

% The problem of the Normalizing flow is the slow inversion of the convolution layers. In this paper we present a fast and invertible convolutional layer for padded input. Our proposed FInC Flow models show consistent improvements over various datasets using the same parameter budget, making the inversion faster when considering models constrained in bigger batch size and input image size. FInC Flow maintains significant sampling time differences with existing flow models such as Emerging, MaCow, Woodbury, and SNF. The generated samples show an unprecedented amount of semantic coherence and exactness of details even at the small scale size of full 8-bit 64 × 64 images.

%% file: paper_tex/appendix.tex
\section*{Additional details}
Code to our implementation is available here:   \url{https://github.com/aditya-v-kallappa/FInCFlow}
\noindent
\paragraph{Datasets}
% with citations
We train our model on standard benchmark datasets MNIST \citep{deng2012mnist}, CIFAR-10 \citep{krizhevsky2009learning}, and ImageNet \citep{russakovsky2015imagenet} sampled down to $32 \times 32$ and $64 \times 64$. We also train our model on CelebA \citep{liu2015deep} sampled down to $64 \times 64$. Figure \ref{fig:celeba_reconstruct}a present the CelebA-$64 \times 64$ reconstructed samples and Figure \ref{fig:uncurated_samples}b for the CIFAR-10 and ImageNet-64x64 generated samples.

\paragraph{Hyperparameters}
To train our model, we use Adam optimizer \citep{kingma2014adam} with learning rate of 0.001 with an exponential decay of 0.99997 per epoch. For training on Imagenet, we also make sure that the gradients stay between $-1$ and $+1$ by clipping them. 

\paragraph{Masking}
To make sure the masked values in the Padded Conv Blocks, we ensure they are not affected by back propagation. To achieve this, we reset the gradients of the masked values to zero after every training iteration.

\paragraph{Cuda code details}
To run CUDA code, we use PyTorch-LTS 1.8.2 and cudatoolkit 10.2. While we can implement Algorithm \ref{alg:finc_flow_inverse} to find inverse of each padded block individually, we can take advantage of the fact that all 4 Padded Conv Blocks are equivalent after proper flipping of padded inputs and kernels. This is done by using Algorithm \ref{alg:finc_flow_inverse} on GPU.

\noindent
First we split $Y$ into 4 parts across channel dimension following the architecture shown in 2. This is given in Line 1. Then we flip $Y$ and kernels to match TL-padding which is given in Line 2. Then we concatenate them to get the final $Y$ and $K$ respectively which are given in Lines 3 and 4. Then we apply Algorithm \ref{alg:finc_flow_inverse} to find the inverse. We do the reverse process of the above to get the correct $X$. The steps are given in Lines 5, 6, 7, 8.

\noindent
We fix the number of threads of each grid of the GPU to be 1024, the number of $grids_x$ to be the batch size and $grids_y$ to be 4 which is the number of Padded Conv Blocks in a single FInC Flow Unit. This ensures that not only GPU inverts the whole FInC Flow Unit at once but also on a batch of images. 

\begin{algorithm}[!h]
    \caption{Fast Parallel Inversion Algorithm for FInC Flow \emph{Unit}}
    \label{alg:finc_flow_inverse}
    \KwIn{$K_1, K_2, K_3, K_4$ - Convolution Kernels of different PCB, $Y$ -  Output of the FInC Flow \emph{Unit}}
    \KwResult{$X$ -  Input to the FInC Flow Unit / Inverse of the FInC Flow Unit}
    \begin{enumerate}
        \item{$Y_1, Y_2, Y_3, Y_4 \gets split(Y)$ }
        \item{Flip $Y_2, Y_3, Y_4, K_2, K_3, K_4$ (inplace) appropriately to match \emph{TL} padding}
        \item{$X \gets concat(Y_1, Y_2, Y_3, Y_4)$}
        \item{$K \gets concat(K_1, K_2, K_3, K_4)$}
        \item{Apply Algorithm \ref{alg:inv_algo} with input $K, Y$  to get $X$}
        \item{$X_1, X_2, X_3, X_4 \gets split(X)$}
        \item{Flip $X_2, X_3, X_4$ appropriately to get the correct output }
        \item{$X \gets concat(X_1, X_2, X_3, X_4$)}
    \end{enumerate}
\end{algorithm}

\paragraph{Running MaCow, Emerging, CInC Flow, SNF}
For MaCow and SNF, we use the official code released by the authors. Emerging was implemented in PyTorch by the authors of SNF. We make use of that. We have implemented CInC Flow on PyTorch to get the results.

\paragraph{Computing Run-time and Confidence Intervals}
We run the model (both forward and sampling) $11$ times and ignore the $1^{\text{st}}$ run as it includes the initialization time. We calculate the mean, standard deviation and $95\%$ confidence interval and plot the numbers. 
\\
To calculate forward time, we pass $100$ images, as for sampling times, we sample $100$ images. We present these numbers in Table-\ref{table:params}.
\\
For sampling time comparison of different models shown in Figure \ref{fig:timing}, we set the total number of parameters for all the models to be close to $5\mathbf{M}$ to make it a fair comparison. 

\paragraph{Hardware/Training Time}
Our hardware setup consists of Intel Xeon E5-2640 v4 processor providing 40 cores, 80 GB of DDR4 RAM, 4 Nvidia GeForce GTX 1080 Ti GPUs each with $12$ GB of VRAM. We train our model on all GPUs using PyTorch's Data Parallel class. We implement early stopping mechanism for smaller datasets like MNIST, CIFAR-10. For others we train the model for a maximum epochs. To evaluate Forward Time and Sampling Time, we use only one of the GPUs.We evaluate our FInC Flow model on MNIST, CIFAR-10, Imagenet-32x32, and Imagenet-64x64 datasets for three metrics - (a) Loss expressed in Bits per Dimension (BPD), (b) Forward Pass Time (FT): Time taken for 100 images to be passed through the model and (c) Sampling Time(ST): Time taken by the model to generate 100 images.  To do this, we train our model with Adam Optimizer with a learning rate ($lr$) of $0.001$ and exponentially reduce the $lr$ by $0.99997$ after each epoch. We have used 4 NVIDIA GTX 1080 Ti GPUs to train our model. Evaluation(FT/ST) is done on a single GPU.